\documentclass[conference]{IEEEtran}

\usepackage[utf8]{inputenc}
\usepackage{amsmath,amssymb,amsthm}
\usepackage{graphicx}
\usepackage{cite}
\usepackage{booktabs}
\usepackage{tabularx}
\usepackage[table]{xcolor}
\usepackage{enumitem}
\usepackage{url}
\usepackage{hyperref}
\usepackage{balance}
\newtheorem{theorem}{Theorem}
\newtheorem{corollary}{Corollary}
\newtheorem{definition}{Definition}

\begin{document}

\title{The Competence Shadow: Theory and Bounds of AI Assistance in Safety Engineering}

\author{\IEEEauthorblockN{Umair Siddique}
	\IEEEauthorblockA{Independent Research \\
		Ottawa, ON, Canada}}

\maketitle

\begin{abstract}

As AI assistants become increasingly integrated into safety engineering workflows, a critical question emerges: does AI assistance improve the quality of safety analysis, or does it introduce systematic blind spots that accumulate invisibly and surface only through post-deployment incidents? This paper develops a formal framework for governing AI assistance in safety analysis as a core component of the broader development and certification of Physical AI. We first establish why safety engineering resists the benchmark-driven evaluation approaches that have enabled rapid AI progress in domains such as software development. Unlike tasks with objective ground truth, safety competence is irreducibly multidimensional, constrained by context-dependent correctness, inherent incompleteness, and legitimate expert disagreement. We formalize this structure through a five-dimensional competence framework capturing domain knowledge, standards expertise, operational experience, contextual understanding, and judgment.

We introduce the competence shadow: the systematic narrowing of human reasoning induced by AI-generated safety analysis, whereby the AI's partial competence profile limits which safety issues are hypothesized, which scenarios are evaluated, and which mitigations are retained. The shadow is not what the AI presents, but what it prevents from being considered. We formalize four canonical human-AI collaboration structures and derive closed-form performance bounds for each, demonstrating that the competence shadow compounds multiplicatively across various cognitive  mechanisms to produce degradation far exceeding naive additive estimates.

The central finding is that AI assistance in safety engineering is a collaboration design problem, not a software procurement decision. The same tool degrades or improves analysis quality depending entirely on how it is used, not how capable it is. We derive non-degradation conditions that are central to developing shadow-resistant safety engineering workflows, and call for a paradigm shift from tool qualification toward workflow qualification as the foundation for accelerated development of trustworthy and certifiable Physical AI.

\end{abstract}

\section{Introduction}
\label{sec:introduction}

Consider a company developing humanoid robots for elder care in residential environments.\footnote{We use humanoid robots as a running example for concreteness. The analysis and frameworks developed here apply broadly to any physical AI system whose deployment requires safety certification, including autonomous vehicles, mobile robots, medical devices and drones.} Their software team uses large language models (LLMs) for code generation and documentation~\cite{chen2021evaluating, peng2023impact}. Product managers use AI to draft user stories and acceptance criteria. The company leadership observes these successes and asks a natural question: if other technical disciplines can accelerate through AI assistance, why not safety engineering?
Some tool vendors and consulting firms in the safety analysis space reinforce this expectation, claiming 60 to 80 percent reductions in certification effort through  \emph{AI-Powered Safety Analysis}.

The question facing safety teams is not whether to integrate AI assistance but how to do so responsibly. Comprehensive hazard analysis for a robot operating in unstructured human environments requires synthesizing requirements from multiple safety standards~\cite{iso13482, iso12100, iec61508, iso13849} while addressing deployment contexts with fundamentally different risk profiles. Home environments with isolated elderly users present different hazards than institutional settings with trained staff. The same failure mode that causes minor inconvenience in one context could result in serious injury in another. Safety analysis is not a one-time deliverable but an iterative process that guides architecture decisions and requires sufficient verification before independent assessment bodies can certify the system.

This context makes one question especially important: does AI assistance improve the \textit{quality} of safety analysis, or can it introduce systematic gaps that remain invisible until post-deployment incidents reveal them? When a safety engineer reviews an AI-generated analysis as their starting point, does it enhance their ability to identify hazards comprehensively, or does it anchor their thinking to the AI's initial framing? And when management observes that AI can draft preliminary analyses in minutes, will they maintain the time allocations that enable engineers to think deeply about edge cases?

Recent empirical work provides important data points. Since 2024, the industry has seen a surge in capability-first studies demonstrating that LLMs can generate industry-standard safety artifacts, including Fault Tree Analysis (FTA)~\cite{shetiya2026fta, shentu2025fta}, Failure Mode and Effects Analysis (FMEA) for automotive systems~\cite{elhassani2025fmea, singh2025fmea}, and Hazard and Operability (HAZOP) reports for process industries~\cite{elhosary2025hazop}. While these studies report significant efficiency gains, they typically rely on empirical methodologies to measure accuracy while leaving the underlying human-AI interaction unformalized. Diemert and Weber~\cite{diemert2023coha} found that 64\% of LLM-generated hazard analysis responses contained useful information. Collier et al.~\cite{collier2025risk} demonstrated that LLMs perform adequately at creative hazard identification but struggle with numerically grounded risk assessment. Qi et al.~\cite{qi2025hazop} evaluated four leading LLMs on HAZOP automation and found that the proportion of semantically valid scenarios remained between 0.19 and 0.37. Bharadwaj et al.~\cite{bharadwaj2025hallucinations} revealed significant variability across safety-critical hazard categories. Charalampidou et al.~\cite{charalampidou2024hazard} found that approximately half of ChatGPT-4-generated unsafe control actions in STPA required expert correction. These findings indicate a technology that is genuinely useful yet inconsistent.

The cognitive science literature offers a crucial lens. Parasuraman and Manzey~\cite{parasuraman2010complacency} established that automation complacency and bias arise from fundamental attentional mechanisms, not mere carelessness, consistent with earlier empirical evidence from Skitka et al.~\cite{skitka1999automation} and subsequently synthesized in a systematic review by Goddard et al.~\cite{goddard2012automation}. Tversky and Kahneman~\cite{tversky1974judgment} established that initial estimates anchor subsequent reasoning, even when adjustment is warranted. Romeo and Conti~\cite{romeo2025automation}, Horowitz and Kahn~\cite{horowitz2024automation}, and Bansal et al.~\cite{bansal2019beyond} have documented automation bias and miscalibrated mental models as drivers of human-AI team underperformance. Dell'Acqua et al.~\cite{dellacqua2023jagged} demonstrated through a field experiment with 758 consultants that AI assistance degraded performance by 19 percentage points on tasks outside the AI's capability frontier, providing direct evidence that collaboration structure determines whether AI helps or harms. 
Chen et al.~\cite{chen2025interface} showed that interface design choices in high-stakes settings shape whether collaboration enhances or degrades performance. Gao et al.~\cite{gao2024taxonomy} have called for systematic frameworks to govern these interactions. Yet no formal theory connects these cognitive phenomena to the specific structure of safety engineering tasks.

We make three contributions toward closing this gap:

\begin{enumerate}[leftmargin=*]
	\item A \textbf{five-dimensional competence framework} that formalizes safety engineering competence and explains why this domain resists the benchmark-driven evaluation that has enabled rapid AI progress in software engineering.

	\item A \textbf{theory of the competence shadow} that identifies four mechanisms through which AI-generated output degrades human safety reasoning and shows that these mechanisms compound multiplicatively.

	\item \textbf{Formal collaboration structures with performance bounds.} We define four canonical human-AI collaboration structures, derive closed-form performance bounds for each, and establish non-degradation conditions that determine when AI assistance is safe to deploy. From these bounds we derive practical guidance for structure selection, safety assessment, and organizational governance. The framework complements emerging standards efforts such as ISO/IEC~TS~22440-1 Annex~C~\cite{isoiec22440}, which addresses AI tool qualification but lacks a model of how AI outputs cast a competence shadow over human cognition.
\end{enumerate}

\section{The Five-Dimensional Competence Framework}
\label{sec:competence}

\subsection{From Scalar Expertise to a Competence Vector}

In software development, competence can be rigorously benchmarked. Code either compiles or it does not. Tests either pass or they break. Performance either meets specifications or it falls short. This objective verifiability explains why LLMs have achieved remarkable success on coding tasks, with datasets like HumanEval~\cite{chen2021evaluating} and Mostly Basic Programming Problems (MBPP)~\cite{austin2021program} providing concrete benchmarks with verifiable correctness metrics.

Safety engineering lacks this collapse to objective outcomes. Safety analyses rarely admit a single correct answer that can be automatically verified, and the multiple dimensions of safety competence remain irreducibly distinct because there is no single ground truth to which they converge. We model safety competence as a five-dimensional vector:
\begin{equation}
	\mathbf{C} = \langle D, S, E, C, J \rangle
	\label{eq:competence}
\end{equation}

\noindent Table~\ref{tab:competence-dimensions} describes each dimension. 

\begin{table}[htbp]
	\centering
	\caption{Dimensions of Safety Engineering Competence}
	\label{tab:competence-dimensions}
	\footnotesize
	\begin{tabular}{@{}p{1.2cm}p{6.3cm}@{}}
		\toprule
		\textbf{Dim.} & \textbf{Description}                                                                                                                                                                                                \\
		\midrule
		\rowcolor{blue!5}
		\textbf{D}    & \textit{\bf Domain knowledge.} System physics, component interactions, failure propagation mechanisms, and domain-specific engineering principles.                                                                      \\
		\addlinespace[3pt]
		\textbf{S}    & \textit{\bf Standards expertise.} Knowledge of applicable safety standards (e.g., IEC~61508~\cite{iec61508}, ISO~26262~\cite{iso26262}, ISO~13482~\cite{iso13482}), compliance procedures, and regulatory approval processes. \\
		\addlinespace[3pt]
		\rowcolor{blue!5}
		\textbf{E}    & \textit{\bf Operational experience.} Accumulated knowledge from debugging production systems, investigating incidents, and observing how systems actually fail rather than how theory predicts they should.             \\
		\addlinespace[3pt]
		\textbf{C}    & \textit{\bf Contextual understanding.} Recognition of how the same technical failure has different severity depending on deployment environment, user population, and available mitigations.                            \\
		\addlinespace[3pt]
		\rowcolor{blue!5}
		\textbf{J}    & \textit{\bf Judgment.} Risk calibration developed through empirical feedback loops between predictions and outcomes. Ability to assess severity and likelihood based on how predictions map to real-world consequences. \\
		\bottomrule
	\end{tabular}
\end{table}

These dimensions are conceptually distinct, though empirically correlated. A compliance specialist may memorize ISO~26262 requirements (high $S$) without understanding automotive control systems deeply enough to recognize novel failure modes (low $D$). An academic researcher may develop strong theoretical domain knowledge (high $D$) while having limited exposure to how systems fail in deployment (low $E$). The competence vector captures variation that a single scalar ``expertise level'' would obscure.

No single safety professional dominates across all dimensions.\footnote{Some engineers do develop broad strength across multiple dimensions, often by spending years as domain experts before moving into safety roles or the reverse. But these profiles reflect careers of deliberate accumulation, and organizational processes cannot be designed around their availability.} Figure~\ref{fig:competence-profiles} illustrates an example how different specialists exhibit distinct competence profiles, and how a team combining complementary profiles achieves comprehensive coverage.

\begin{figure}[htbp]
	\centering
	\includegraphics[width=\columnwidth]{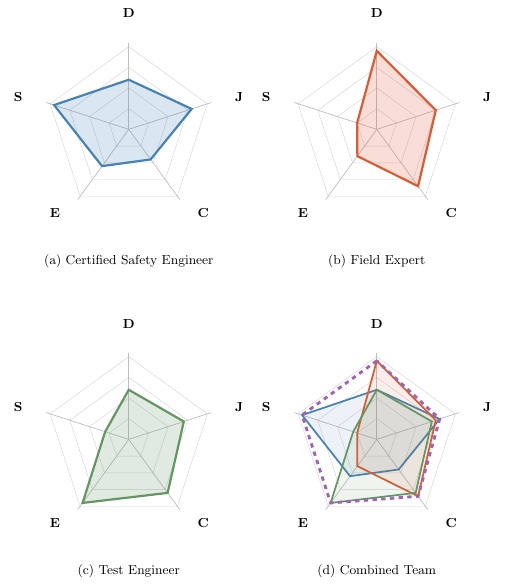}
	\caption{Complementary competence profiles forming a complete team. Panel~(a): consider a certified safety engineer, strong in standards (S) and judgment (J), limited in operational exposure. Panel~(b): field expert, strong in domain knowledge (D) and contextual understanding (C). Panel~(c): test engineer, strong in operational experience (E) and contextual knowledge (C). Panel~(d): combined team coverage (dashed envelope).}
	\label{fig:competence-profiles}
\end{figure}

\subsection{Fundamental Barriers to Benchmarking}

The multidimensional structure clarifies why safety competence resists benchmarking. Three fundamental barriers reflect the intrinsic nature of safety knowledge, not limitations of current measurement techniques.

\textbf{Context-dependent ground truth.} The same failure mode demands 
different severity ratings across deployment contexts, and no 
context-independent standard exists to adjudicate between them. 
An unexpected arm motion in a humanoid robot is LOW severity 
in a supervised industrial assembly cell, HIGH in a rehabilitation 
clinic where patients have limited mobility, and CRITICAL in a 
home environment where an elderly user with dementia may not 
recognize or respond to warning signals. Each rating is correct 
within its context.

\textbf{Inherent incompleteness.} Whether the analysis identifies \textit{all} failures that could lead to harm cannot be verified prior to deployment. Novel failure modes reveal themselves only through incidents, often years after design decisions are finalized. 
Incident databases that could reveal these gaps are typically inaccessible outside organizations due to liability concerns and intellectual property protections.

\textbf{Legitimate expert disagreement.} Five experienced engineers analyzing the same system will construct five different fault trees employing different decomposition strategies, operating at different granularity levels, and bringing different experiential priors. This disagreement reflects the irreducibly perspectival nature of safety judgment, not a deficiency to be eliminated.

These barriers also explain why ``qualifying'' an AI-assistance based  safety analysis tool under
existing frameworks (ISO~26262 Part~8, IEC~61508 Part~3) is insufficient: tool qualification
addresses deterministic correctness, not the cognitive dynamics of human-AI
interaction. Safety engineering workflows designed for deterministic 
systems are increasingly strained by the complexity and pace of modern 
autonomous systems, and several frameworks have been proposed to 
modernize safety lifecycle practices~\cite{siddique2020safetyops, carlan2025safety}.
However, none formally address how AI assistance itself shapes the 
quality of the resulting analysis, and the cognitive mechanisms through 
which this occurs remain unformalized. 
Recent findings from industry reports confirm that while LLMs can brainstorm a wide volume of hazards, their outputs are frequently described as generic or lacking the depth of a subject matter expert~\cite{elhassani2025fmea}.
Critically, the gap in operational experience~($E$), contextual 
understanding~($C$), and judgment~($J$) is not primarily an 
information deficit addressable by retrieval-augmented generation 
or knowledge graphs: these dimensions are formed through sustained 
feedback loops between predictions and real-world consequences, 
and whether they can be meaningfully approximated in AI systems 
remains an open research challenge distinct from the collaboration 
design problem this paper addresses.

\subsection{Implications for AI-Assisted Safety Engineering}

The inability to rigorously benchmark AI-generated safety analysis 
is not a temporary limitation of current tools but reflects the 
irreducible structure of safety competence itself. Yet teams are 
increasingly relying on LLMs to generate fault trees and populate 
FMEA tables, often without formal understanding of which competence 
dimensions the AI can and cannot contribute to. The resolution is 
not to qualify AI assistants in isolation but to design collaboration 
structures that deliberately match AI capabilities to task requirements 
along specific competence dimensions. This requires a formal model 
of how AI outputs interact with human reasoning, which is what the 
remainder of this paper develops.

\section{The Competence Shadow}
\label{sec:shadow}

Knowing which competence dimensions an AI system lacks does not, by itself, explain the risks of AI-assisted safety analysis. The deeper issue lies in how AI outputs interact with human cognition during the analysis process itself. We term this the \textit{competence shadow}: the AI's partial competence profile casts a shadow over the human analyst's reasoning, systematically narrowing which hazards are hypothesized, which scenarios are explored, and which findings are retained. The shadow is not what the AI presents, but what it prevents from being considered. When an engineer reviews an AI-generated hazard analysis, four mechanisms produce this shadow. These mechanisms are consistent with the automation bias phenomenon established by Skitka et al.~\cite{skitka1999automation} and subsequently confirmed across professional domains~\cite{parasuraman2010complacency}.

\textbf{Mechanism 1: Scope Framing ($\alpha_{\text{frame}}$).} AI-generated analysis establishes an implicit ontology of what constitutes a relevant failure mode. Engineers working from this frame readily identify hazards fitting the AI's taxonomy, but failure modes requiring alternative decomposition strategies become cognitively harder to generate~\cite{tversky1974judgment}. This framing effect is consistent with Green and Chen's~\cite{green2019principles} finding that algorithmic recommendations anchor human judgment even when presented as advisory. Consider the elder-care robot: an AI trained on general robotics literature might frame hazards around mechanical failure and collision avoidance, while entirely missing the interaction patterns specific to cognitively impaired users in home environments. The AI's frame may be technically sound yet incomplete in ways reflecting its competence gaps in $E$ and $C$.

\textbf{Mechanism 2: Attention Allocation Bias ($\beta$).} Engineers reviewing AI output face an implicit resource allocation decision: spend time verifying what the AI found, or search for what it may have missed? In practice, verification dominates: reviewing AI output for errors is bounded, offers predictable returns, and produces visible progress, while independent exploration is open-ended and uncertain. The general automation bias literature confirms that operators working with automated systems consistently prioritize monitoring system output over independent analysis~\cite{parasuraman2010complacency, goddard2012automation}. We estimate that this dynamic leads engineers to allocate approximately 60 to 70 percent of effort to verification, leaving only 30 to 40 percent for independent exploration.

\textbf{Mechanism 3: Confidence Asymmetry ($\eta_{\text{disagree}}$).} When engineers identify a failure mode that the AI also identified, concordance functions as confirmatory evidence and retention probability approaches unity. When engineers identify a hazard \textit{absent} from the AI's output, cognitive dissonance arises: their judgment stands against the implicit judgment of a system that has processed far more safety analyses than any individual. The natural response is self-doubt. This dynamic mirrors Bansal et al.'s~\cite{bansal2019beyond} finding that miscalibrated mental models of AI capability systematically degrade team performance. This asymmetry creates differential retention probabilities that systematically favor issues within the AI's competence profile.

\textbf{Mechanism 4: Organizational Time Compression ($\gamma$).} When management observes AI completing preliminary analyses in minutes, organizational pressure to compress safety analysis timelines intensifies. In safety engineering, this dynamic creates a time compression ratchet (Section~\ref{sec:practice}) that we model through the parameter $\gamma \in (0, 1]$. Time compression directly degrades baseline human capability ($q_{h,\text{eff}} = \gamma \cdot q_h$) and amplifies all three cognitive mechanisms: reduced time forces greater reliance on verification, raises retention thresholds for AI-contradicting findings, and makes independent reasoning outside the AI's frame prohibitive.

These four mechanisms operate through distinct cognitive and organizational channels, yet their joint effect is multiplicative rather than additive: each mechanism scales the residual human capability left by those preceding it. How severe this joint effect becomes depends on the structure of the collaboration itself, which we formalize next.

\section{Collaboration Structures and Performance Bounds}
\label{sec:bounds}

We formalize four collaboration structures that make fundamentally different commitments about information flow and task decomposition. These commitments determine which shadow mechanisms are active and therefore bound the quality of the resulting analysis.

Let $S = \{s_1, \ldots, s_m\}$ be the complete set of safety-critical issues. Define quality as $Q = |S_{\text{identified}}|/|S|$, and let $q_h$ and $q_{AI}$ denote human and AI baseline identification probabilities.

\subsection{Structure Definitions}

\begin{definition}[Serial Dependency, $\pi_1$]
	\label{def:serial}
	The AI generates an initial analysis $S_{AI} \subseteq S$. Human reviewers observe $S_{AI}$ in full before conducting their review. Information flows from AI to human during the analysis phase. The final analysis is $S_{\emph{final}} = S_{AI} \cup S_{\emph{human\text{-}review}}$. All four shadow mechanisms are structurally active: scope framing ($\alpha_{\emph{frame}}$), attention allocation ($\beta$), confidence asymmetry ($\eta_{\emph{disagree}}$), and time compression ($\gamma$).
\end{definition}

\begin{definition}[Independent Analysis and Synthesis, $\pi_2$]
	\label{def:independent}
	All agents ($k$ human analysts and one AI system) perform analysis independently without observing each other's output. No information flows between agents during the analysis phase. A designated lead analyst then reconciles the independent results through a structured synthesis phase, resolving duplicates and adjudicating conflicts. The final analysis is $S_{\emph{final}} = S_{AI} \cup \bigcup_{i=1}^{k} S_{h_i}$. Because humans never see AI output during analysis, three shadow mechanisms are structurally eliminated ($\alpha_{\emph{frame}} = \beta = \eta_{\emph{disagree}} = 1$). Only time compression ($\gamma$) remains, though typically at reduced severity.
\end{definition}

\begin{definition}[Tool Augmentation, $\pi_3$]
	\label{def:tool}
	Human analysts perform all core safety reasoning. The AI is confined to auxiliary tasks: formatting, compliance cross-referencing, template population, and documentation. Information flows from human to AI in the form of auxiliary queries only. The final analysis is $S_{\emph{final}} = S_{\emph{human\text{-}core}} \cup S_{\emph{AI\text{-}aux}}$, where $S_{\emph{AI\text{-}aux}}$ contributes only to presentation, not to safety content. No shadow mechanisms are active on core analysis, provided a clean decomposition boundary is maintained between core reasoning tasks (requiring $\langle D, E, C, J \rangle$) and auxiliary tasks (requiring primarily $\langle S \rangle$). Boundary errors occur with probability $\varepsilon$, affecting a fraction $\delta$ of safety issues.
\end{definition}

\begin{definition}[Human-Initiated Exploration, $\pi_4$]
	\label{def:hie}
	The human analyst performs an initial analysis $S_h$ independently, without AI involvement. The human then provides $S_h$ to the AI, requesting identification of gaps, alternative propagation paths, or additional failure modes. Information flows from human to AI and back to human during a revision phase. The final analysis is $S_{\emph{final}} = S_h \cup S_{\emph{AI\text{-}new}}$, where $S_{\emph{AI\text{-}new}}$ denotes novel findings the AI contributes beyond the human's starting point. Because the human establishes a clean-room analysis before AI engagement, scope framing ($\alpha_{\emph{frame}}$) and attention allocation ($\beta$) are structurally eliminated. Confidence asymmetry ($\eta_{\emph{disagree}}$) remains active during the revision phase, and time compression ($\gamma$) applies to the initial manual phase.
\end{definition}

Figure~\ref{fig:protocols-overview} illustrates the structural differences in information flow across all four structures.

\begin{figure}[htbp]
	\centering
	\includegraphics[width=\columnwidth]{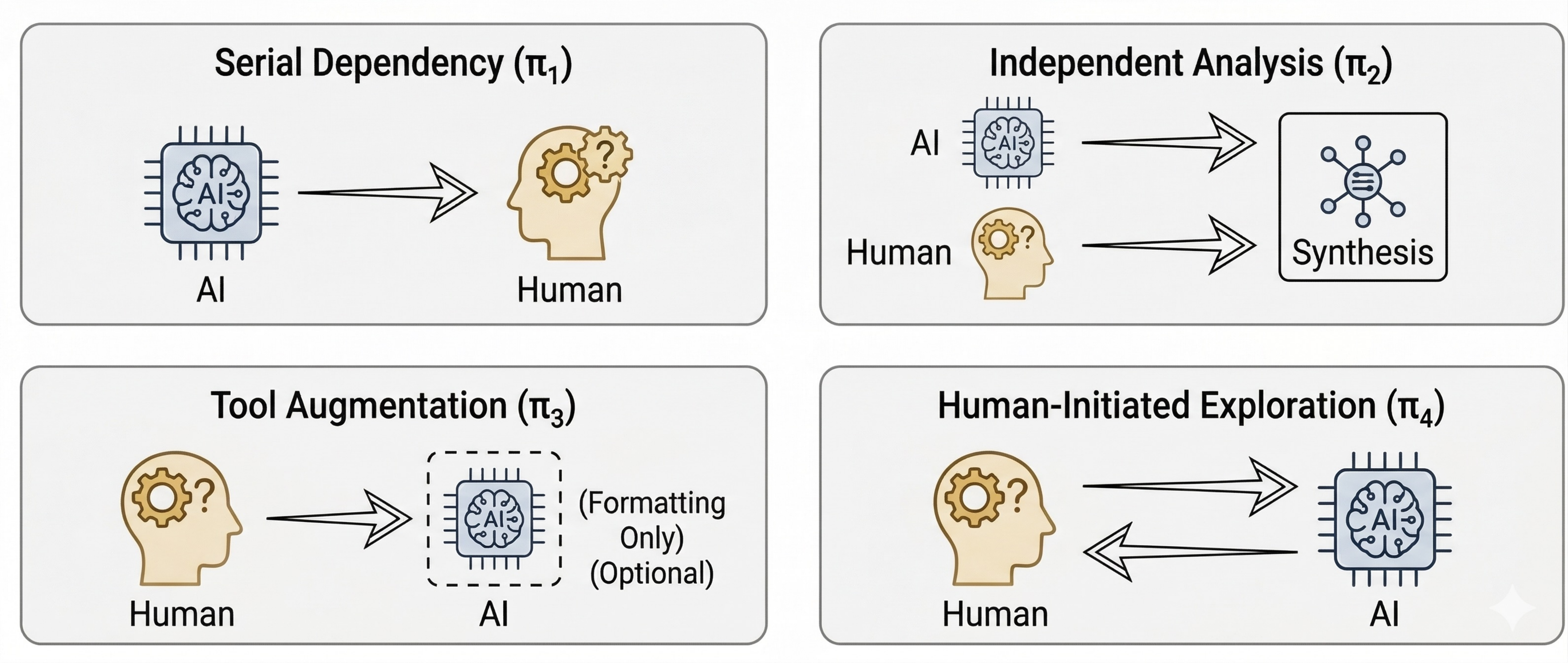}
	\caption{Four canonical human-AI collaboration structures for AI-assisted safety analysis, differing in information flow and active shadow mechanisms.}
	\label{fig:protocols-overview}
\end{figure}

\subsection{Serial Dependency: The Compounding Shadow}

Serial Dependency ($\pi_1$) is the unexamined default in current safety engineering research. Advanced frameworks, such as the Aegis multi-agent system~\cite{shi2024aegis} and pre-populated HAZOP tables~\cite{elhosary2025hazop}, operate exclusively under this structure. It is the dominant deployment pattern in practice, driven by its apparent efficiency: the AI does the heavy lifting, and humans ``check the work''. 
Our theory proves that $\pi_1$ is the structure most susceptible
to the multiplicative compounding of the competence shadow, and the
significant epistemic risk it carries is precisely what its apparent
efficiency conceals. Because all four shadow mechanisms are
active under $\pi_1$, we first formalize their compound effect.

\begin{definition}[Effective Anchoring Coefficient]
	\label{def:alpha_eff}
	Under Serial Dependency ($\pi_1$) with scope framing $\alpha_{\emph{frame}}$, attention allocation $\beta$, and confidence asymmetry $\eta_{\emph{disagree}}$, the effective anchoring coefficient is:
	\begin{equation}
		\alpha_{\emph{eff}} = \alpha_{\emph{frame}} \cdot \beta \cdot \eta_{\emph{disagree}}
		\label{eq:alpha_eff}
	\end{equation}
\end{definition}

The parameters used throughout this section are illustrative 
values chosen to represent plausible moderate-shadow conditions, 
consistent with the automation bias literature~\cite{parasuraman2010complacency, 
goddard2012automation, horowitz2024automation} and with patterns 
observed in AI-assisted knowledge work~\cite{dellacqua2023jagged}. 
We assume scope framing reduces the accessible failure space to 
approximately 80\% of its full extent ($\alpha_{\text{frame}} = 0.8$), 
that roughly 70\% of review effort goes to verification rather than 
independent exploration ($\beta = 0.3$, reflecting 30\% retained 
for exploration), that engineers retain approximately 70\% of 
AI-discordant findings ($\eta_{\text{disagree}} = 0.7$), and that 
time compression reduces available analysis time by 40\% 
($\gamma = 0.6$). These values are not empirically measured in 
safety engineering workflows; the qualitative conclusions hold 
across a wide range of parameter values, and developing 
standardized methodologies for measuring shadow parameters 
is the most significant research opportunity 
this framework opens for the community (Section~\ref{sec:discussion}).

With these representative parameters, the effective anchoring 
coefficient is $\alpha_{\text{eff}} = 0.168$, meaning a 
shadow-affected reviewer retains only 16.8\% of their independent 
identification capability before time compression is applied.

\begin{theorem}[Serial Dependency with Compounding Shadow]
	\label{thm:serial}
	Under $\pi_1$ with compounding shadow and $k$ independent human reviewers:
	\begin{equation}
		\mathbb{E}[Q(\pi_1)] = q_{AI} + (1 - q_{AI}) \cdot \left[1 - (1 - \alpha_{\emph{eff}} \cdot \gamma \cdot q_h)^k\right]
	\end{equation}
	where $\alpha_{\emph{eff}}$ is from~\eqref{eq:alpha_eff} and $\gamma \in (0,1]$ is the time compression ratio.
\end{theorem}

\begin{proof}
	For issue $s \in S$: if AI identifies $s$ (probability $q_{AI}$), humans retain it (probability 1 under perfect verification). If AI misses $s$ (probability $1 - q_{AI}$), each human independently identifies it with probability $\alpha_{\text{eff}} \cdot \gamma \cdot q_h$, reflecting scope framing, attention allocation, confidence asymmetry, and time-compressed baseline capability. The issue is identified if at least one reviewer finds it: $P(s \text{ identified} \mid s \notin S_{AI}) = 1 - (1 - \alpha_{\text{eff}} \cdot \gamma \cdot q_h)^k$. By the law of total probability, the result follows.
\end{proof}

\begin{corollary}[Non-Degradation Condition]
	\label{cor:nondeg}
	Serial Dependency with a single reviewer preserves quality relative to human baseline ($\mathbb{E}[Q(\pi_1)] \geq q_h$) if and only if:
	\begin{equation}
		q_{AI} \geq \frac{q_h(1 - \alpha_{\emph{eff}} \cdot \gamma)}{1 - \alpha_{\emph{eff}} \cdot \gamma \cdot q_h}
	\end{equation}
\end{corollary}

\textbf{Numerical illustration.} Suppose $q_h = 0.85$ and $q_{AI} = 0.65$. Under realistic shadow conditions ($\alpha_{\text{eff}} = 0.168$, $\gamma = 0.6$): $\mathbb{E}[Q(\pi_1)] = 0.65 + 0.168 \times 0.6 \times 0.85 \times 0.35 = 0.68$. This is a \textbf{20\% degradation} from human baseline, despite AI assistance. The non-degradation threshold requires $q_{AI} \geq 0.74$.

Figure~\ref{fig:waterfall} visualizes the degradation structure, tracing the sequential impact of each mechanism.

\begin{figure}[htbp]
	\centering
	\includegraphics[width=\columnwidth]{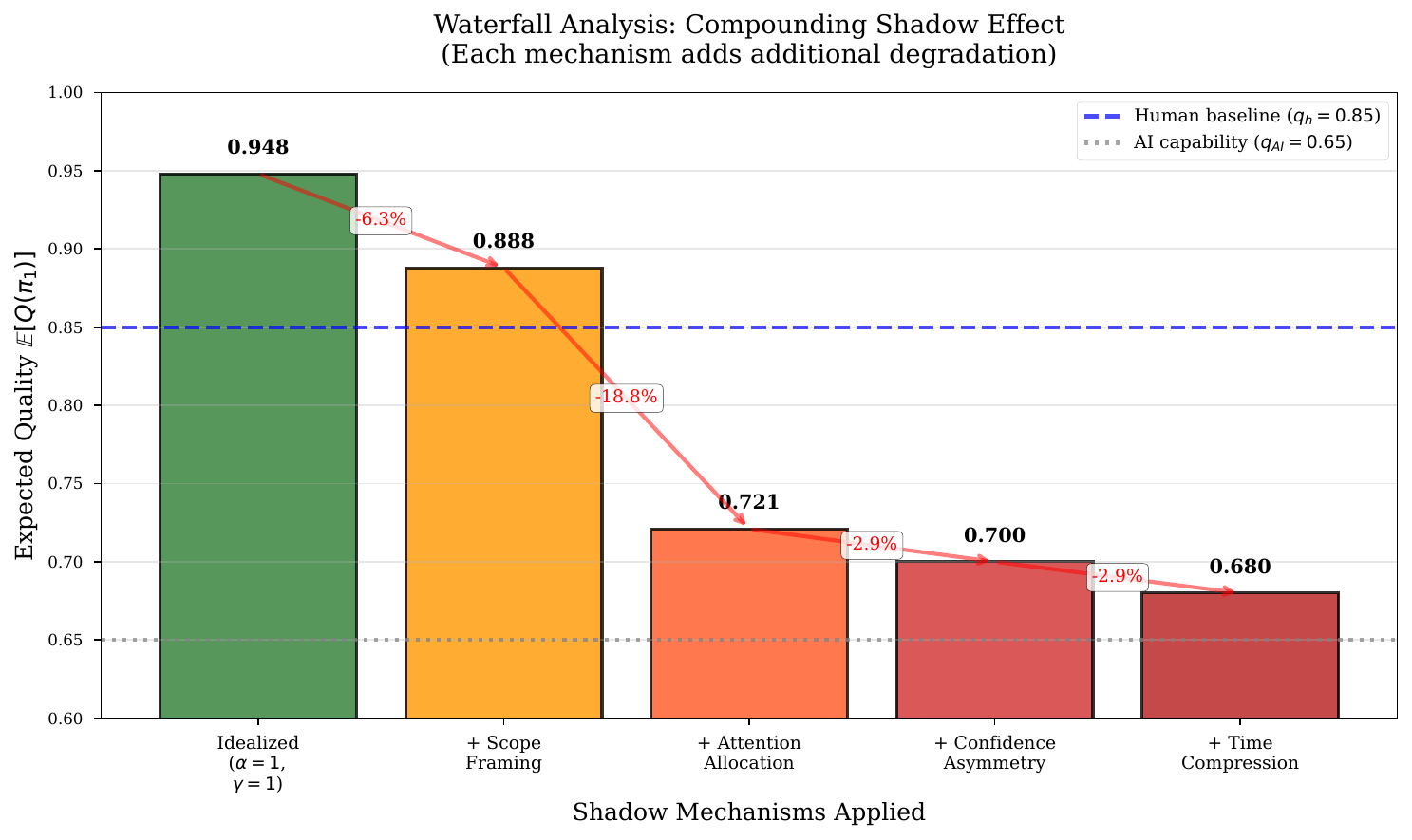}
	\caption{Waterfall analysis of the compounding shadow. Starting from the idealized case at 0.948, each mechanism sequentially reduces quality: scope framing (0.888), attention allocation (0.721), confidence asymmetry (0.700), and time compression (0.680). Final quality is 20\% below human baseline.}
	\label{fig:waterfall}
\end{figure}

\subsection{Independent Analysis: Eliminating the Shadow by Construction}

Independent Analysis ($\pi_2$) prevents information flow during the analysis phase, eliminating the competence shadow by construction rather than by discipline. This structural approach aligns with Bu\c{c}inca et al.'s~\cite{bucinca2021trust} finding that cognitive forcing functions significantly reduce 
over-reliance on AI in assisted decision-making.
\begin{theorem}[Independent Analysis Performance]
	\label{thm:independent}
	Under $\pi_2$ with $k$ humans, AI capability $q_{AI}$, and structured correlation $\rho$ with shared blind spot probability $q_{\emph{shared}}$:
	\begin{equation}
		\mathbb{E}[Q(\pi_2)] = 1 - \rho \cdot q_{\emph{shared}} - (1 - \rho)(1 - q_h)^k(1 - q_{AI})
	\end{equation}
\end{theorem}

\noindent This structure remains vulnerable only to time compression ($\gamma$), typically at reduced severity ($\gamma \approx 0.85$ versus $\gamma \approx 0.6$ under Serial Dependency) because its design makes the case for adequate time allocation more legible to management.

\textbf{Numerical illustration.} With $q_h = 0.85$, $q_{AI} = 0.65$, $k = 3$, $\rho = 0.3$, $q_{\text{shared}} = 0.4$: $\mathbb{E}[Q(\pi_2)] = 1 - 0.12 - 0.7 \times (0.15)^3 \times 0.35 \approx 0.88$. This is \textbf{3 percentage points above human baseline} and \textbf{20 points above shadow-affected Serial Dependency}.

\subsection{Tool Augmentation: Clean Decomposition}

Tool Augmentation ($\pi_3$) confines AI to tasks where the competence shadow cannot form. The critical requirement is a clean decomposition boundary: core tasks requiring $\langle D, E, C, J \rangle$ must be separated from auxiliary tasks requiring primarily $\langle S \rangle$.

\begin{theorem}[Tool Augmentation Performance]
	\label{thm:tool}
	Under $\pi_3$ with clean decomposition and boundary error probability $\varepsilon$ affecting fraction $\delta$ of issues:
	\begin{equation}
		\mathbb{E}[Q(\pi_3)] = q_h \cdot (1 - \varepsilon \cdot \delta)
	\end{equation}
\end{theorem}

\noindent With $\varepsilon = 0.03$, $\delta = 0.5$: $\mathbb{E}[Q(\pi_3)] = 0.85 \times 0.985 \approx 0.837$, representing only 1.5\% degradation from baseline while achieving 30\% time savings on auxiliary work. When decomposition fails ($\varepsilon$ rises to 0.15), quality drops to 0.80.

\subsection{Human-Initiated Exploration: Reversing the Information Flow}

Human-Initiated Exploration ($\pi_4$) addresses the structural gap between Serial Dependency's risky efficiency and Independent Analysis's resource cost. By requiring the human to complete an initial analysis before engaging the AI, $\pi_4$ temporally segregates the framing mechanism: the human's baseline capability $q_h$ is preserved because no AI output exists to anchor their reasoning during the critical divergent-thinking phase. The AI then functions as a challenger rather than an initiator.

Recent work by Shentu and Trapp~\cite{shentu2025fta} demonstrates a version of $\pi_4$, where a generative AI co-pilot suggests new sub-causes for human-driven fault trees. While this effectively extends the divergent coverage of the analysis, our framework reveals that the effectiveness of this structure is bounded by the AI's propensity to anchor on the human's initial system decomposition.

\begin{theorem}[Human-Initiated Exploration Performance]
	\label{thm:hie}
	Under $\pi_4$ with human baseline $q_h$, AI capability $q_{AI}$, reverse anchoring factor $\rho_{\emph{rev}} \in [0,1]$, and acceptance rate $\eta_{\emph{accept}}$ for AI-suggested findings:
	\begin{equation}
		\mathbb{E}[Q(\pi_4)] = q_h + (1 - q_h) \cdot \eta_{\emph{accept}} \cdot (1 - \rho_{\emph{rev}}) \cdot q_{AI}
	\end{equation}
	where $\rho_{\emph{rev}}$ captures the degree to which exposure to the human's initial analysis $S_h$ constrains the AI's exploration. At $\rho_{\emph{rev}} = 0$ the AI explores freely and may identify failure modes entirely outside the human's frame; at $\rho_{\emph{rev}} = 1$ the AI is fully anchored on the human's decomposition and contributes nothing new.
\end{theorem}

\begin{proof}
	For issue $s \in S$: if the human identifies $s$ in the clean-room phase (probability $q_h$), it is retained. If the human misses $s$ (probability $1 - q_h$), the AI may identify it independently with probability $(1 - \rho_{\text{rev}}) \cdot q_{AI}$, reflecting its base capability discounted by the degree of reverse anchoring, and the human accepts the finding with probability $\eta_{\text{accept}}$. Because the human's initial analysis is unshadowed, $q_h$ enters without degradation. By the law of total probability, the result follows.
\end{proof}

\textbf{Numerical illustration.} With $q_h = 0.85$, $q_{AI} = 0.65$, $\rho_{\text{rev}} = 0.30$ (moderate reverse anchoring: the AI is partially constrained by the human's frame but retains substantial independent exploration capability), and $\eta_{\text{accept}} = 0.70$: $\mathbb{E}[Q(\pi_4)] = 0.85 + 0.15 \times 0.70 \times 0.70 \times 0.65 = 0.898$. This is \textbf{5 percentage points above human baseline}, positioning $\pi_4$ between Serial Dependency (68\%) and Independent Analysis (88\%) under moderate reverse anchoring, and competitive with $\pi_2$ under favorable conditions.

Table~\ref{tab:comparison} presents the complete comparison under consistent parameters.

\begin{table}[htbp]
	\centering
	\caption{Four-structure comparison ($q_h=0.85$, $q_{AI}=0.65$)}
	\label{tab:comparison}
	\small
	\begin{tabular}{@{}lcccc@{}}
		\toprule
		\textbf{Metric}       & \textbf{Serial} & \textbf{Indep.} & \textbf{Tool Aug.} & \textbf{HIE} \\
		\midrule
		Expected Quality      & 68\%            & 88\%            & 84\%               & 90\%         \\
		$\Delta$ vs. Baseline & $-17$ pp        & $+3$ pp         & $-1$ pp            & $+5$ pp      \\
		Shadow Mechs. Active & 4/4             & 1/4             & $\varepsilon$ only & 1/4          \\
		\bottomrule
	\end{tabular}
\end{table}

\section{From Bounds to Practice}
\label{sec:practice}

\subsection{Structure Selection for Safety Teams}

The non-degradation condition (Corollary~\ref{cor:nondeg}) gives safety teams a structural decision rule: Serial Dependency preserves analysis quality only when AI capability exceeds a threshold determined by the severity of the competence shadow in that specific workflow. For moderate to strong shadow effects, the required AI capability is substantially higher than what current LLMs demonstrate, suggesting that for most safety analysis tasks today Serial Dependency is likely to degrade quality.

The practical response is to match the structure to the task. A single project involves dozens of distinct analysis activities, and different activities warrant different structures depending on which competence dimensions they demand. The key recommendation is to decompose safety workflows at the activity level and assign structures deliberately, rather than adopting a single collaboration pattern for the entire project. This decomposition is itself a design decision that deserves review, because an incorrect boundary between ``core reasoning'' and ``auxiliary task'' introduces exactly the semantic leakage that degrades Tool Augmentation performance (Theorem~\ref{thm:tool}).

Organizations must also resist the \textit{time compression ratchet}. When AI completes preliminary analyses in minutes, management has a rational basis for tightening schedules, and when shadow effects remain invisible in the output, the ratchet tightens again. The non-degradation condition provides a principled floor: the formal model yields the minimum time allocation below which quality drops beneath human baseline, and treating this allocation as freely negotiable is itself a safety-critical decision. 

\subsection{Implications for Safety Assessment}

Independent safety assessors face a new problem: when an organization submits a hazard analysis for certification review, the assessor must now evaluate not only whether the analysis is technically adequate but whether the process that produced it was epistemically sound. Two FMEA tables can look identical on paper while carrying very different epistemic weight. One may reflect thorough human reasoning augmented by AI-drafted formatting; the other may reflect AI-generated failure modes accepted under scope framing without independent exploration. From the output alone, these scenarios are indistinguishable.

The formal framework suggests three questions that assessors can use to target competence shadow risk:

\begin{enumerate}[leftmargin=*]
	\item \textbf{Which structure governed each analysis activity?} The structure determines which shadow mechanisms were active (Table~\ref{tab:comparison}), and Serial Dependency carries fundamentally different epistemic risk than Independent Analysis.

	\item \textbf{Does AI capability meet the non-degradation threshold?} If Serial Dependency was used, the assessor should ask whether AI performance on comparable tasks exceeds the bound from Corollary~\ref{cor:nondeg}. Serial Dependency without evidence of sufficient AI capability is a risk that should be justified explicitly.

	\item \textbf{What is the epistemic provenance of the safety claims?} When the majority of failure modes trace back to AI-generated content with only editorial human review, the effective quality is bounded by $q_{AI}$ rather than $q_h$.
\end{enumerate}

\section{Discussion and Open Problems}
\label{sec:discussion}

\subsection{Challenged Assumptions}

\textbf{Assumption 1: AI capability alone determines value.} The quality gap between Serial Dependency and Independent Analysis (Table~\ref{tab:comparison}) arises entirely from collaboration structure, consistent with Dell'Acqua et al.'s~\cite{dellacqua2023jagged} field experimental finding that the same AI tool improved performance on some tasks while degrading it on others depending entirely on how it was used. Organizations investing in better models while ignoring structure design may be degrading their safety analyses.

\textbf{Assumption 2: Efficiency gains are unqualified benefits.} When management reduces safety analysis time from 40 to 24 hours ($\gamma = 0.6$), they compound the competence shadow multiplicatively. The resulting quality degradation is invisible in the output, inviting further schedule tightening. As the community moves toward engineering safe and sustainable computing systems, the non-degradation bounds derived here provide the formal grounding: without shadow-resistant structures like Independent Analysis ($\pi_2$), the efficiency gains promised by current capability-first research~\cite{shetiya2026fta, elhassani2025fmea, elhosary2025hazop} risk being offset by the invisible competence shadow.

\textbf{Assumption 3: Tool qualification frameworks are sufficient for AI assistants.} Existing standards and emerging efforts such as ISO/IEC~TS~22440-1 Annex~C~\cite{isoiec22440} model AI as producing outputs that are correct or incorrect, then ask whether humans can catch the errors. Our framework reveals a deeper problem: the most dangerous failure mode is not incorrect AI output but \textit{correct-looking} output that casts a competence shadow over human cognitive performance. An AI-generated hazard list that is plausible but incomplete is not ``malfunctioning'' under current definitions, yet it systematically reduces the reviewer's ability to identify what is missing. More fundamentally, current frameworks lack a concept of collaboration structure and implicitly assume a single workflow (AI generates, human reviews) without recognizing that structurally different arrangements produce dramatically different quality outcomes.

\subsection{Open Problems}

\textbf{Shadow parameter estimation.} Our model treats $\alpha_{\text{frame}}$, $\beta$, $\eta_{\text{disagree}}$, and $\gamma$ as given parameters, but no standardized methodology exists to measure them. This paper provides the theoretical framework and formal structure; empirical calibration is the essential next step. We call for controlled studies in which safety engineers perform hazard analysis under each structure condition, with effort allocation, retention rates, and independent discovery rates measured directly. 

\textbf{Dynamic shadow evolution.} We treat shadow parameters as static, but real deployments involve learning. We conjecture a U-shaped dynamic: an initially strong shadow from novelty, followed by compensation as engineers develop metacognitive awareness, potentially followed by complacency. Formalizing these dynamics requires multi-period models with learning parameters.

\section{Conclusion: Toward Human-Centric Safety Intelligence}
\label{sec:conclusion}

A safety engineer reviewing an LLM-generated fault tree is not 
simply ``checking AI's work.'' They are operating within a cognitive 
environment shaped by the AI's framing, where independent judgment 
is systematically suppressed through mechanisms they cannot 
consciously resist.

We argue for \textit{Human-Centric Safety Intelligence}: the 
principled integration of AI capabilities within collaboration 
structures that preserve and amplify the irreplaceable dimensions 
of human expertise. LLMs genuinely excel at standards-intensive, 
documentation-heavy tasks. The error is in deploying them without 
understanding the cognitive architecture of the resulting human-AI 
system.

Software engineering has HumanEval~\cite{chen2021evaluating}
and MBPP~\cite{austin2021program} to measure AI capability. Safety engineering has nothing comparable. 
Building on these foundations requires
a coordinated research program. We call for: (1)~\textit{structure-aware
evaluation benchmarks} that measure what the human-AI system finds
under specified structure conditions; (2)~\textit{shadow parameter 
databases} where organizations contribute anonymized measurements 
from deployed workflows; (3)~\textit{longitudinal field studies} 
tracking whether engineers develop shadow resistance or complacency; 
and (4)~\textit{standards evolution} extending tool qualification 
frameworks to address collaboration structure classification, shadow 
dynamics, and documentation requirements calibrated to structure risk.

Millions of robots, vehicles, and medical devices will be deployed 
over the coming decade. Whether those systems are genuinely safe 
depends on choices being made today. Safety-critical AI assistance 
is achievable, but it requires shadow-resistant collaboration 
structures, not merely capable models.

\bibliographystyle{IEEEtran}
\bibliography{conference_refs}

\end{document}